\documentclass{article}

\usepackage[preprint,nonatbib]{neurips_2023}

\usepackage[utf8]{inputenc} 
\usepackage[T1]{fontenc}    
\usepackage{hyperref}       
\usepackage{url}            
\usepackage{booktabs}       
\usepackage{amsfonts}       
\usepackage{nicefrac}       
\usepackage{microtype}      
\usepackage{xcolor}         
\usepackage{physics}
\usepackage{bm}
\usepackage{bbm}
\usepackage{algorithm}
\usepackage{algpseudocode}
\usepackage{amsthm}
\newcommand{\reals}{\mathbb{R}}
\newcommand{\bx}{\mathbf{x}}
\newcommand{\by}{\mathbf{y}}

\DeclareMathOperator*{\argmin}{argmin}

\newtheorem*{theorem*}{Theorem}
\newtheorem{lemma}{Lemma}

\newtheorem{example}{Example}

\title{Mirror Diffusion Models}

\author{
  Jaesung Tae\\
  Department of Computer Science\\
  Yale University\\
  \texttt{jake.tae@yale.edu} \\
}

\begin{document}

\maketitle

\begin{abstract}
Diffusion models have successfully been applied to generative tasks in various continuous domains. However, applying diffusion to discrete categorical data remains a non-trivial task. Moreover, generation in continuous domains often requires clipping in practice, which motivates the need for a theoretical framework for adapting diffusion to constrained domains. Inspired by the mirror Langevin algorithm for the constrained sampling problem, in this theoretical report we propose Mirror Diffusion Models (MDMs). We demonstrate MDMs in the context of simplex diffusion and propose natural extensions to popular domains such as image and text generation.
\end{abstract}

\section{Introduction}

Diffusion models have achieved state-of-the-art in various domains, such as images~\cite{Ho2020DenoisingDP,Song2020ScoreBasedGM}, audio~\cite{Kong2020DiffWaveAV,Popov2021GradTTSAD,Tae2021EdiTTSSE}, and video~\cite{Ho2022VideoDM}. Diffusion models typically involve a forward process, which gradually corrupts the input signal into noise.\footnote{Here, we only consider standard diffusion models which use Gaussian noise. However, we note that there exist alternative formulations of diffusion models which defines the corruption process by any arbitrary operator, such as Cold Diffusion~\cite{Bansal2022ColdDI}.} The stationary distribution of the forward process defines the prior, which serves as a starting point of the reverse process. The model iteratively refines its estimate to produce a sample from the desired data distribution.

Despite the notable success of diffusion models in continuous domains, applying diffusion models to sample from discrete categorical distributions remains a challenging task. The motivation for modeling categorical distributions is clear: for instance, natural language generation can be formulated as a discrete sampling problem, where the goal is to sample from a joint distribution of word tokens, which are inherently discrete. Recently, SSD-LM~\cite{han2022ssdlm} proposed to perform diffusion on the logit probability simplex, which involves shift-scaled representation of one-hot token vectors. However, the design choices for logit simplex diffusion is largely guided by heuristics, and there is a lack of analysis of the specific mechanics behind the validity of its formulation.

In this work, we propose Mirror Diffusion Models (MDMs), which draws inspiration from mirror maps, a common method employed for constrained optimization and sampling. Specifically, we provide a theoretical account of logit simplex diffusion and show that existing methods lack a correction term that is required to obey the manifold structure of the constrained primal space. We ground the demonstration of MDMs in the context of natural language generation, and also propose natural extensions that generalizes simplex diffusion to other domains. 

\section{Background}

In this section, we briefly revisit diffusion models and mirror Langevin dynamics. More in-depth expositions on diffusion models can be found in~\cite{Song2020ScoreBasedGM,Ho2020DenoisingDP,Kingma2021VariationalDM}; for mirror Langevin dynamics and mirror descent for constrained optimization,~\cite{li2021mirror,Jiang2021MirrorLM,Ahn2020EfficientCS,Krichene2017AccelerationAA}.

\subsection{Diffusion}

While various formulations of diffusion models exist, e.g., Denoising Diffusion Probabilistic Models (DDPMs)~\cite{Ho2020DenoisingDP}, Variational Diffusion Models (VDMs)~\cite{Kingma2021VariationalDM}, and Score-based Generative Models (SGMs)~\cite{Song2020ScoreBasedGM}, here we consider the latter, which is based on stochastic differential equations (SDEs). We believe that these formulations are largely interchangeable and are different ways of characterizing a similar family of generative models.

Diffusion models typically involve a forward process, which gradually corrupts the input signal, and a reverse process, where the model iteratively refines its prediction to recover the original signal. The forward diffusion process is defined by an Itô SDE of the form

\begin{equation}
\dd \bx = \mathbf{f}(\bx, t) \dd t + g(t) \dd \mathbf{w}, 
\end{equation}

where $\dd \mathbf{w}$ denotes the standard Bronwian motion and $\mathbf{f}: \reals^d \cross \reals \to \reals^d$ and $g: \reals \to \reals$ are drift and diffusion coefficients, respectively. The corresponding reverse SDE can be derived by considering evolution of the density under the partial differential equation described by the Fokker-Planck equation. Concretely, let $\bar{\mathbf{w}}$ denote the reverse-time Brownian motion. Then the reverse SDE~\cite{Anderson1982ReversetimeDE} is given by

\begin{equation}
\dd \bx = (\mathbf{f}(\bx, t) - g(t)^2 \nabla_\bx \log p_t(\bx)) \dd t + g(t) \dd \bar{\mathbf{w}},
\label{eq:reverse-sde}
\end{equation}

Note that the reverse SDE involves the Stein score function, which is key to modeling the generative process. Diffusion models are typically trained through the denoising score-matching objective, in which the model learns to estimate the score function at every time step $t \in \{1, 2, \dots, T \}$ given a finite time horizon $T$. This discretization is required to implement the forward and backward SDEs via Euler-Maruyama. Concretely, let $\lambda_t$ be a weighting function. The denoising score-matching objective~\cite{Vincent2011ACB,Song2020ScoreBasedGM} is given by

\begin{equation}
\mathcal{L}_\text{DSM} = \mathbb{E}_t \left\{ \lambda_t \mathbb{E}_{\bx_0} \mathbb{E}_{\bx_t | \bx_0}  \left[ \norm{s_{\bm{\theta}}(\bx_t, t) - \nabla \log p(\bx_t | \bx_0)}^2_2 \right] \right\}.
\end{equation}

In practice, the forward SDE is typically discretized into a Gaussian Markov chain. For instance, the variance-preserving SDE in DDPM is given by

$$
\bx_{k + 1} = \sqrt{1 - \beta_{k + 1}} \bx_k + \sqrt{\beta_{k + 1}} \mathbf{z}_{k},
$$

which is a discretization of the continuous-time SDE

$$
\dd \bx = - \frac12 \beta(t) \bx \, \dd t + \sqrt{\beta(t)} \dd \mathbf{w}.
$$

With this construction, given a model that can approximate the score, i.e., $\mathbf{s}_\theta(\bx, t) \approx \nabla \log p_t(\bx)$, we can sample from the stationary distribution of the forward process and simulate the reverse SDE to sample from the learned distribution $p_\text{data}(\bx)$. 

\subsection{Mirror Langevin}

Suppose we want to sample from $\nu \propto e^{-f}$ on $\Omega = \reals^d$. In continuous time, we can use Langevin dynamics, which involves the standard Brownian motion in $\reals^d$:

\begin{equation}
\dd \bx = - \nabla f(\bx) \, \dd t + \sqrt{2} \, \dd \mathbf{w}.
\end{equation}

The seminal work by Jordan-Kinderlehrer-Otto~\cite{Jordan1996THEVF} established the remarkable interpretation of Langevin dynamics as a gradient flow in the Wasserstein space of probability measures with the objective function $D_\text{KL}(\nu \parallel \rho_0)$, where $\rho_0$ denotes the prior distribution such that $\bx_0 \sim \rho_0$. The dynamics of the SDE in the PDE space is well-understood by the Fokker-Planck equation.

However, a one limitation of Langevin dynamics is that it is ineffective for constrained sampling. In many applications, we often want to sample from a distribution with bounded support. In image generation---the most common application of diffusion models---sampled pixel values must lie between $[0, 255]$ to represent RGB colors with the standard 8-bit quantization. Similar constraints arise in waveform or mel-spectrogram generation. Lastly, recent works attempt to perform diffusion on the $(d - 1)$-dimensional standard simplex $\triangle_d = \{ \bx \in \reals^d \mid \sum_{i = 1}^d x_i = 1, x_i \geq 0 \}$ to model discrete categorical generative processes. 

A recent development in constraint sampling is the mirror Langevin algorithm (MLA), which draws inspiration from mirror descent in constrained optimization. Let $\phi: \Omega \to \reals$ be a twice-differentiable function that is of Legendre type, and $\nabla \phi: \Omega \to \reals^d$ denote the mirror map. Recall that the inverse mirror map $\nabla \phi^*: \reals^d \to \Omega$ is given by the gradient of the Lengendre-Fenchel conjugate of $\phi$. Then the update rule for MLA~\cite{li2021mirror,Jiang2021MirrorLM} is given by

\begin{equation}
\bx_{k + 1} = \nabla \phi^* (\nabla \phi(\bx_k) - h \nabla f(\bx_k) + \sqrt{2 h \nabla^2 \phi (\bx_k)} \mathbf{z}_k),
\label{eq:mla}
\end{equation}

MLA can be seen as a Euler-Maruyama discretization of the continuous-time mirror Langevin dynamics:

\begin{equation}
\begin{cases}
\by &= \nabla \phi(\bx) \\
\dd \by &= - \nabla f(\bx) \dd t + \sqrt{2 \nabla^2 \phi(\bx)} \dd \mathbf{w}.
\end{cases}
\label{eq:mld}
\end{equation}

We also note the connection between mirror Langevin and the standard Langevin, which can be established by considering a specific $\phi$.

\begin{example}
(Identity Mirror Map). Let $\norm{\cdot}$ denote the usual norm in $\reals^d$. Let $\phi = \frac12 \norm{\bx}^2$ and $\Omega = \reals^d$. Then $\nabla \phi (\bx) = \bx$, $\phi^* = \frac12 \norm{\bx}^2$, and mirror Langevin dynamics recovers the standard Langevin dynamics in $\Omega$. Likewise, the mirror Langevin algorithm reduces to the Unadjusted Langevin Algorithm (ULA).
\end{example}

Intuitively, mirror Langevin methods reformulate the constrained sampling problem by transporting particles to the dual space, which is no longer bound by the original constraint in the primal space, then leveraging the inverse mirror map to transport variables in the dual space back to the desired primal space.

\subsection{Logit Simplex Diffusion}

Simplex diffusion is a concrete instance of the constrained sampling problem in which the bounded domain is given by the $(d - 1)$-dimensional standard simplex $\triangle_d = \{ \bx \in \reals^d \mid \sum_{i = 1}^n x_i = 1, x_i \geq 0 \}$. Simplex diffusion has potential for wide application, particularly in modeling discrete categorical distributions, whose parameters are naturally supported on the probability simplex by construction. Simplex diffusion has emerged as a promising way of circumventing the continuity-discreteness gap, which is the primary bottleneck in applying existing diffusion models to sample from categorical distributions.

While many variants of simplex diffusion exist, in this work, we focus on the logit simplex diffusion method proposed by SSD-LM~\cite{han2022ssdlm}, a diffusion-based language model. 
Recall that the goal of language modeling is to train a model that learns the joint probability distribution over a sequence of words. In SSD-LM, the forward process occurs on a shift-scaled representation of one-hot vectors of word tokens. Specifically, for each word token, we construct a one-hot vector $\bx$, where the $i$-th token is represented with a sparse vector of all 0's except its $i$-th column, which holds a value of 1. These one-hot vectors are then shift-scaled to produce an output $\by$, which is given by

\begin{equation}
y_{i} = 
\begin{cases}
k & x_{i} = 1 \\
- k & \text{otherwise.} \\
\end{cases}
\label{eq:shiftscale}
\end{equation}

The forward diffusion process occurs on this shift-scaled representation. The model is then trained to estimate the score function in the Euclidean space. At the end of the reverse process, a softmax operation is applied to project the predicted $\hat{\by}$ back onto the probability simplex.\footnote{In practice, each step of the reverse process also involves an argmax and $k$-shift scaling. This detail does not affect our analysis, and we focus on the fact that a softmax is applied at the very end.} Let $d$ denote the cardinality of the vocabulary space, i.e., $d = |\mathcal{V}|$. Then the softmax operation on $\hat{\by}$ is given by

\begin{equation}
\hat{\bx} = \text{softmax}(\hat{\by}) = \left( \frac{e^{\hat{y}_1}}{\sum_{i = 1}^d e^{\hat{y}_j}}, \frac{e^{\hat{y}_2}}{\sum_{i = 1}^d e^{\hat{y}_j}}, \dots, \frac{e^{\hat{y}_n}}{\sum_{i = 1}^d e^{\hat{y}_j}} \right).
\label{eq:softmax}
\end{equation}

The softmax does not produce a one-hot vector, but a categorical distribution over the vocabulary space. Therefore, we can apply existing sampling techniques for natural language generation, such as top-$k$ or top-$p$ nucleus sampling with beam search, which produces a sequence of word tokens, as desired.

This construction effectively transforms the constrained sampling problem on the simplex to an unconstrained sampling problem in Euclidean space $\reals^d$. Intuitively, the shift-scaling process relaxes the constrained domain from the simplex to the Euclidean space, in which any SDE can be simulated. The reverse process also occurs in Euclidean space, and the final softmax ensures that the predicted output is projected back to the simplex.

\section{Mirror Diffusion Models}

Inspired by mirror Langevin methods, we propose Mirror Diffusion Models (MDMs) for constrained sampling. We first outline the structure of mirror diffusion models in relations to MLA and existing diffusion models, then consider its concrete application to simplex diffusion for modeling sequences of categorical variables.

\subsection{Method}

Consider the constrained sampling problem, i.e., we want to sample from a probability distribution $\nu \propto e^{-f}$ supported on a convex set $\Omega \subset \reals^d$. Let $\phi$ be a twice-differentiable strictly convex function which is of the Legendre type. We call $\nabla \phi: \Omega \to \reals^d$ the mirror map to the dual space.

The core difference between MDM and standard diffusion models is that in MDMs, we differentiate between the primal and dual space. In standard diffusion models, such as SGMs, both the sampling distribution and the prior distribution lie in Euclidean space, i.e., the distributions are supported on $\reals^d$. In contrast, in MDMs, the desired sampling distribution is supported on $\Omega$, and only the prior lies on $\reals^d$. 

For concreteness, consider the case when $\Omega = \triangle$. We first map observed data points $\bx_0$ on the probability simplex to the dual space via the mirror map, i.e., $\by_0 = \nabla \phi (\bx_0)$, and run forward diffusion in the dual space according to mirror Langevin dynamics. Then, a score-based generative model is trained to learn the score function in the dual space so that we can simulate the reverse SDE, which eventually yields $\hat{\by}_0 \approx \by_0$. The mechanics of training an MDM is equivalent to that of standard diffusion models. This process is demonstrated in Algorithm~\ref{alg:training}.

Once an MDM is trained, we first start from the prior distribution in the dual space, then use the score estimates from the model to simulate the reverse SDE. Combining Equations~\ref{eq:reverse-sde} and~\ref{eq:mld}, the reverse SDE for mirror Langevin dynamics can be written as

\begin{equation}
\dd \by 
= (- \nabla f(\bx) - 2 \nabla^2 \phi(\bx) \nabla \log p(\by)) \dd t + \sqrt{2 \nabla^2 \phi(\bx)}\dd \bar{\mathbf{w}}
\label{eq:mdm-reverse}
\end{equation}

Note that the mirror map and its inverse provides a direct mapping from $\bx$ to $\by$ and vice versa, so the reverse SDE is tractable. Like standard SGMs, the reverse SDE can be discretized via Euler-Maruyama and simulated over a discrete finite time horizon $T$. After $T$ denoising steps, the final estimate is mapped back to the primal space via the inverse mirror map, i.e., $\hat{\bx}_0 = \nabla \phi^* (\hat{\by}_0) \approx \bx_0$. This process is schematically shown in Algorithm~\ref{alg:sampling}.

\begin{figure}[t]
\begin{minipage}[t]{0.495\textwidth}
\begin{algorithm}[H]
  \caption{Training} \label{alg:training}
  \small
  \begin{algorithmic}[1]
    \Repeat
      \State $\bx_0 \sim q(\bx_0)$
      \State $\by_0 \gets \nabla \phi(\bx_0)$
      \State $t \sim \mathrm{Uniform}(\{1, \dotsc, T\})$
      \State $\mathbf{z} \sim \mathcal{N}(0,\mathbb{I})$
      \State Take gradient descent step on
      \Statex $\qquad \grad_\theta \| \mathbf{z} - \mathbf{f}_\theta(\tilde{\mathbf{z}}, t) \|^2$
    \Until{converged}
  \end{algorithmic}
\end{algorithm}
\end{minipage}
\hfill
\begin{minipage}[t]{0.495\textwidth}
\begin{algorithm}[H]
  \caption{Sampling} \label{alg:sampling}
  \small
  \begin{algorithmic}[1]
    \vspace{.04in}
    \State $\by_T \sim \mathcal{N}(0, \mathbb{I})$
    \For{$t=T, \dotsc, 1$}
      \State $\mathbf{z} \sim \mathcal{N}(0,\mathbb{I})$
      \State $\hat{\mathbf{s}} \gets \mathbf{s}_\theta(\by_t, t)$
      \State $\by_{t - 1} \gets$ \texttt{solve\_sde}($\by_t, \hat{\mathbf{s}}, \mathbf{z}$)
    \EndFor
    \State \textbf{return} $\nabla \phi^*(\by_0)$
  \end{algorithmic}
\end{algorithm}
\end{minipage}
\label{alg:overview}
\end{figure}

\subsection{Simplex Diffusion}

A natural application of MDM is sampling from the probability simplex. Mirror descent and mirror Langevin algorithms have been well-studied in the simplex case. In particular, the simplex nicely pairs with the negative entropy function, which is 1-strongly convex by Pinsker's inequality~\cite{Krichene2017AccelerationAA}.

\begin{lemma}
(Pinsker's Inequality). For any distributions $p$ and $q$, we have $D_\text{KL}(p \parallel q) \geq \frac12 \norm{p - q}^2_1$.
\label{le:pinsker}
\end{lemma}

\begin{lemma}
(1-Convexity of Negative Entropy). The negative entropy function $\phi(\bx) = \sum_{i = 1}^d x_i \ln x_i$ is 1-convex with respect to the L1 norm.
\end{lemma}

\begin{proof}
Using the definition of Bregman divergence and Pinsker's inequality in Lemma~\ref{le:pinsker}, we have

\begin{align*}
\phi(\by) 
&\geq \phi(\bx) + \langle \nabla \phi(\bx), \by - \bx \rangle + D_\text{KL}(\by \parallel \bx) \\
&\geq \phi(\bx) + \langle \nabla \phi(\bx), \by - \bx \rangle + \frac12 \norm{\by - \bx}^2_1.
\end{align*}
\end{proof}

Intriguingly, the mirror map defined by the negative entropy function bears resemblance to the shift-scale operators in SSD-LM.

\begin{example}
(Negative Entropy). The gradient of the negative entropy function defines a mirror map $\nabla \phi(\bx) = 1 + \ln \bx$ from $\triangle_d$ to $\reals^d$. The Legendre-Fenchel conjugate of $\phi$ is given by $\phi^*(\by) = \sum_{i = 1}^d \exp(y_i - 1)$, and the inverse mirror map (with Bregman projection) is given by the softmax function.
\label{ex1}
\end{example}

We observe that the mirror map in Example~\ref{ex1} is similar to the shift-scale transformation in Equation~\eqref{eq:shiftscale}. Specifically, the mirror map induced by the negative entropy function is also a step-wise function in the binary case:

\begin{equation}
\nabla \phi(\bx)_i = 
\begin{cases}
1 & x_{i} = 1 \\
- \infty & \text{otherwise.} \\
\end{cases}
\label{eq:mirrormap}
\end{equation}

Negative infinity is numerically unstable to use in practice, which justifies the $k$-logit simplex parametrization in Equation~\eqref{eq:shiftscale}. We could also consider the first layer of the neural network as having an affine transformation that maps 1 and $- \infty$ to $k$ and $-k$, respectively.

Furthermore, the inverse mirror map of the negative entropy function is given by the softmax function as in Equation~\eqref{eq:softmax}, which is exactly the same operator used in SSD-LM to project the denoised iterates back into the probability simplex. Therefore, the logit simplex diffusion method in SSD-LM can be seen as an instantiation of an MDM in the simplex case using the gradient of the negative entropy function as the mirror map to $\reals^d$.

\subsection{Applications}

We highlight potential applications of MDMs by demonstrating how existing problems in both discrete and continuous domains can be reformulated as constrained sampling problems.

\paragraph{Categorical Distributions.} Let $\mathbf{c} = \{ c_1, c_2, \dots, c_k \}$ denote $k$ nominal categories. Suppose we model a dataset $\mathbf{X} = \{ \bx_1, \bx_2, \dots, \bx_n \}$, where each $\bx_i$ is a sequence of length $L$ such that $\bx_i \in \mathbf{c}^L$ and the $\bx_i$'s are i.i.d. The goal is to construct a generative model $p_\theta(\bx) = p_\theta(x_1, x_2, \dots, x_t)$. 

To apply MDMs to this generic problem, we consider one-hot representations; that is, we map each $\bx_i$ to $\bx'_i \in \{ 0, 1\}^{k \times L} \subset \triangle_k^L$, where $\triangle_k^L$ denotes a sequence of $L$ $(k - 1)$-dimensional simplices. We then employ the negative entropy function to create a mirror map $\nabla \phi: \triangle_k^L \to \reals^{k \times L}$. Once an MDM is trained on the Euclidean dual space, the inverse mirror map $\nabla \phi^*$ can be applied to map each of the $L$ columns back to $\triangle_k$ after simulating a reverse SDE, thus completing the generative pipeline. Recall that this is the formulation used in SSD-LM.

Note that if each components of $\bx_i$ are independent, the problem can be drastically simplified to modeling a single variable, then invoking the product rule, i.e., $p(\bx) = \prod_{i = 1}^L p(x_i)$. However, in most real-world applications, this independence assumption does not hold. Most problems of interest usually involve some degree of autoregressiveness (future samples depend on past samples) or locality (samples within a close window are more closely correlated with each other than those that are far away), among other possible interactions. Like standard diffusion models, MDMs place no restriction on the architecture of the score estimator $\bm{\theta}$, which makes it tenable for this purpose.

\paragraph{Distributions with Bounded Support.} Even in continuous domains, generation often requires clipping or projecting predicted iterates to some bounded domain. For instance, in image generation, sampled pixel values must lie between $[0, 255]$ to represent RGB colors. Similar constraints arise in waveform generation, which requires values to be bounded in $[-1.0, 1.0]$. Vanilla diffusion models enforce this boundary constraint by constantly projecting the iterate at each step of the reverse SDE~\cite{lou2023reflected}. Generally, the constraint appears in the form of a closed interval in $\reals$ because quantization is applied to encode continuous values with a fixed bit-width on modern computers.

These interval restrictions make MDMs a compelling choice for modeling continuous data. Recall that polytopes are well-studied objects in constrained optimization~\cite{Bubeck2014ConvexOA}. Typically, logarithmic barrier functions are used to ensure that the iterate does not escape the boundary defined by the interior of the polytope. A crucial observation is that common interval constraints on each component of the data can also geometrically be interpreted as a polytope. Any constrained $d$-dimensional cuboid $[a, b]^d$, where $a, b \in \reals$ and $d$ denotes the data dimension (e.g., number of pixels or duration of audio), is effectively a convex polytope, which makes the application of mirror methods possible via a suitable log barrier function. In short, given their generic formulation, MDMs are in theory applicable to domains beyond the probability simplex, with potential to be employed in popular domains such as image and audio generation.

\subsection{Limitations}

\paragraph{Speed.} MDMs build upon the standard SGM paradigm and is thus prone to the same limitations of general diffusion models. In particular, generating samples from diffusion models is known to be slow due to the iterative nature of their sampling process. Improving the sampling speed of diffusion models is an active area of research~\cite{song2023consistency}, and we fully expect these developments to give rise to improved MDM methods.

\paragraph{Hessian.} The reverse SDE in Equation~\eqref{eq:mdm-reverse} involves a Hessian. Although the Hessian is straightforward to calculate in certain cases, such as the simplex and the negative entropy function, this may not always be the case. In particular, the tractability of the Hessian is directly contingent on the definition of the mirror map. Depending on the problem statement, a canonical mirror map may not always exist, and even if it does, there is no guarantee as to its computability. Last but not least, this report is theoretical in nature; we leave rigorous technical experiments for future work.

\section{Related Works}

\subsection{Constrained Sampling via Projection}

One way of adapting Langevin dynamics for constrained sampling is by applying operators that push out-of-bound samples back into the desired domain through methods such as projection or reflection. Here, we discuss two methods that fall under this category.

\paragraph{Projected Langevin Algorithm.} The Projected Langevin Algorithm (PLA)~\cite{Bubeck2015SamplingFA,Bubeck2015FiniteTimeAO} adds an additional projection step to ULA. The PLA update step can be written as

$$
\bx_{k + 1} = \text{proj}_\Omega (\bx_k - \eta \nabla f(\bx_k) + \sqrt{2 \eta} \mathbf{z}_k),
$$

where $\text{proj}_\Omega$ denotes a projection operation onto $\Omega$, and $\eta$ is the step size. The projection ensures that samples stay within the specified domain throughout the sampling process. Specifically, projected samples are always sent to the boundary of $\Omega$.

\paragraph{Reflected SDEs.} Another approach, Reflected Diffusion Models~\cite{lou2023reflected}, draws inspiration from reflected SDEs, which typically take the form

$$
\dd \bx = \mathbf{f}(\bx, t) \dd t + g(t) \dd \mathbf{w} + d \mathbf{l},
$$

where $\mathbf{f}: \Omega \times \reals \to \reals^d$ and $\mathbf{l}$ is the additional boundary constraint that forces the particle to stay in $\Omega$ by neutralizing any outward normal-pointing component. Let $\bar{\mathbf{l}}$ denote the reversed boundary condition. The reverse SDE is then given by 

$$
\dd \bx = (\mathbf{f}(\bx, t) - g(t)^2 \nabla_\bx \log p_t(\bx)) \dd t + g(t) \dd \bar{\mathbf{w}} + \dd \bar{\mathbf{l}}.
$$

The Euler-Maruyama discretization of a reflected SDE is given by

$$
\bx_{t + \Delta t} = \mathcal{O} (\bx_t + f(\bx_t, t) \Delta t + g(t) \mathbf{z}_{\Delta t}),
$$

where $\mathcal{O}$ is some suitable operator that enforces the boundary condition. In practice, reverse SDEs are often  implemented as a projection $\text{proj}_\Omega(x) = \argmin_{y \in \Omega} d(x, y)$. PLA can thus be viewed as a particular discretized realization of a reflected SDE. Another common operator is reflection $\text{refl}_\Omega$, where the normal component that moves the particle out of $\Omega$ is reversed across the boundary surface.

\subsection{Simplex Diffusion via Cox–Ingersoll–Ross}

Simplex diffusion is a specific instance of constrained sampling where the support $\Omega = \triangle_n$. A natural application of simplex diffusion is for modeling discrete categorical variables, such as in language modeling. One approach~\cite{richemond2022categorical} proposes to model simplex diffusion via the Cox-Ingersoll-Ross (CIR) process, whose terminal distribution is known to be the Gamma distribution. CIR is defined as

$$
\dd x_t = \beta(\alpha - x_t) \dd t + \sigma \sqrt{x_t} \dd w_t,
$$

where $\alpha, \beta, \sigma > 0$ are process parameters. The terminal distribution is then $\mathcal{G}(2 \alpha \beta / \sigma^2, 2 \beta / \sigma^2)$. Notably, setting $2 \beta = \sigma^2$ yields $\mathcal{G}(\alpha, 1)$. By simulating $n$ independent CIR processes to sample $\bx = (x_1, x_2, \dots, x_n)$, where each $x_i \sim \mathcal{G}(\alpha_i, 1)$, one can sample from a Dirichlet distribution via

$$
\left(\frac{x_1}{\sum_{i = 1}^n x_i}, \frac{x_1}{\sum_{i = 1}^n x_i}, \dots, \frac{x_n}{\sum_{i = 1}^n x_i} \right) \sim \mathcal{D}(\bm{\alpha}).
$$

If $\bm{\alpha} = \mathbbm{1}_d$, then $\mathcal{D}(\bm{\alpha}) = \mathcal{U}(\triangle_n)$, i.e., the uniform distribution over the simplex. The generative process is then given by the reverse-time SDE. The known results from score-based SDE methods, e.g., reverse SDE, evidence lower bound, and maximum likelihood estimation, remain applicable under this framework.

\section{Conclusion}

We propose Mirror Diffusion Models, which adapts diffusion models for constrained sampling problems using mirror Langevin dynamics. We provide a theoretical account of a recent method in simplex diffusion, which uses shift-scaling and softmax operations to map variables on the simplex to the Euclidean space, and vice versa. We show that these operations can elegantly be understood in the framework of mirror Langevin dynamics, which involves a mirror map that pushes variables on the constrained space forward to the unconstrained dual space, and a corresponding inverse mirror map which pulls dual variables back to the primal space. We demonstrate potential applications of MDMs to modeling of both discrete and bounded continuous distributions.

\acksection
We express gratitude to Andre Wibisono, Jane Lee, Jiaming Liang, and Arman Cohan for continuous feedback, insight, and support throughout the completion of this work.

\bibliographystyle{plain}
\bibliography{refs}

\begin{thebibliography}{10}

\bibitem{Ahn2020EfficientCS}
Kwangjun Ahn and Sinho Chewi.
\newblock Efficient constrained sampling via the mirror-langevin algorithm.
\newblock In {\em Neural Information Processing Systems}, 2020.

\bibitem{Anderson1982ReversetimeDE}
Brian. D.~O. Anderson.
\newblock Reverse-time diffusion equation models.
\newblock {\em Stochastic Processes and their Applications}, 12:313--326, 1982.

\bibitem{Bansal2022ColdDI}
Arpit Bansal, Eitan Borgnia, Hong-Min Chu, Jie Li, Hamideh Kazemi, Furong
  Huang, Micah Goldblum, Jonas Geiping, and Tom Goldstein.
\newblock Cold diffusion: Inverting arbitrary image transforms without noise.
\newblock {\em ArXiv}, abs/2208.09392, 2022.

\bibitem{Bubeck2014ConvexOA}
S{\'e}bastien Bubeck.
\newblock Convex optimization: Algorithms and complexity.
\newblock {\em Found. Trends Mach. Learn.}, 8:231--357, 2014.

\bibitem{Bubeck2015FiniteTimeAO}
S{\'e}bastien Bubeck, Ronen Eldan, and Joseph Lehec.
\newblock Finite-time analysis of projected langevin monte carlo.
\newblock In {\em NIPS}, 2015.

\bibitem{Bubeck2015SamplingFA}
S{\'e}bastien Bubeck, Ronen Eldan, and Joseph Lehec.
\newblock Sampling from a log-concave distribution with projected langevin
  monte carlo.
\newblock {\em Discrete \& Computational Geometry}, 59:757--783, 2015.

\bibitem{han2022ssdlm}
Xiaochuang Han, Sachin Kumar, and Yulia Tsvetkov.
\newblock Ssd-lm: Semi-autoregressive simplex-based diffusion language model
  for text generation and modular control, 2022.

\bibitem{Ho2020DenoisingDP}
Jonathan Ho, Ajay Jain, and P.~Abbeel.
\newblock Denoising diffusion probabilistic models.
\newblock In {\em Neural Information Processing Systems}, 2020.

\bibitem{Ho2022VideoDM}
Jonathan Ho, Tim Salimans, Alexey Gritsenko, William Chan, Mohammad Norouzi,
  and David~J. Fleet.
\newblock Video diffusion models.
\newblock In {\em ICLR 2022 Workshop DGM4HSD}, 2022.

\bibitem{Jiang2021MirrorLM}
Qijia Jiang.
\newblock Mirror langevin monte carlo: the case under isoperimetry.
\newblock In {\em Neural Information Processing Systems}, 2021.

\bibitem{Jordan1996THEVF}
Richard Jordan, David Kinderlehrer, and Felix Otto.
\newblock The variational formulation of the fokker-planck equation.
\newblock {\em Siam Journal on Applied Mathematics}, 1996.

\bibitem{Kingma2021VariationalDM}
Diederik~P. Kingma, Tim Salimans, Ben Poole, and Jonathan Ho.
\newblock Variational diffusion models.
\newblock In {\em Neural Information Processing Systems}, 2021.

\bibitem{Kong2020DiffWaveAV}
Zhifeng Kong, Wei Ping, Jiaji Huang, Kexin Zhao, and Bryan Catanzaro.
\newblock Diffwave: A versatile diffusion model for audio synthesis.
\newblock In {\em ICLR}, 2020.

\bibitem{Krichene2017AccelerationAA}
Walid Krichene and Peter~L. Bartlett.
\newblock Acceleration and averaging in stochastic descent dynamics.
\newblock In {\em NIPS}, 2017.

\bibitem{li2021mirror}
Ruilin Li, Molei Tao, Santosh~S. Vempala, and Andre Wibisono.
\newblock The mirror langevin algorithm converges with vanishing bias, 2021.

\bibitem{lou2023reflected}
Aaron Lou and Stefano Ermon.
\newblock Reflected diffusion models, 2023.

\bibitem{Popov2021GradTTSAD}
Vadim Popov, Ivan Vovk, Vladimir Gogoryan, Tasnima Sadekova, and Mikhail
  Kudinov.
\newblock Grad-tts: A diffusion probabilistic model for text-to-speech.
\newblock In {\em International Conference on Machine Learning}, 2021.

\bibitem{richemond2022categorical}
Pierre~H. Richemond, Sander Dieleman, and Arnaud Doucet.
\newblock Categorical sdes with simplex diffusion, 2022.

\bibitem{song2023consistency}
Yang Song, Prafulla Dhariwal, Mark Chen, and Ilya Sutskever.
\newblock Consistency models.
\newblock In {\em International Conference on Machine Learning}, 2023.

\bibitem{Song2020ScoreBasedGM}
Yang Song, Jascha~Narain Sohl-Dickstein, Diederik~P. Kingma, Abhishek Kumar,
  Stefano Ermon, and Ben Poole.
\newblock Score-based generative modeling through stochastic differential
  equations.
\newblock In {\em ICLR}, 2021.

\bibitem{Tae2021EdiTTSSE}
Jaesung Tae, Hyeongju Kim, and Taesu Kim.
\newblock Editts: Score-based editing for controllable text-to-speech.
\newblock In {\em Interspeech}, 2021.

\bibitem{Vincent2011ACB}
Pascal Vincent.
\newblock A connection between score matching and denoising autoencoders.
\newblock {\em Neural Computation}, 23:1661--1674, 2011.

\end{thebibliography}

\end{document}